\documentclass[a4paper,twoside]{article}
\usepackage{epsfig}
\usepackage{subcaption}
\usepackage{calc}
\usepackage{amssymb}
\usepackage{amstext}
\usepackage{amsmath}
\usepackage{amsthm}
\usepackage{multicol}
\usepackage{pslatex}
\usepackage{natbib}
\let\bibhang\relax
\usepackage{apalike}
\usepackage[bottom]{footmisc}

\usepackage[T1]{fontenc}
\usepackage{graphicx} 
\usepackage{xcolor} 
\usepackage{booktabs} 
\usepackage{multirow} 
\usepackage[autostyle=true]{csquotes} 
\usepackage{subcaption} 
\usepackage{xspace} 
\usepackage{amssymb} 
\usepackage{amsmath} 
\newtheorem{theorem}{Theorem} 

\PassOptionsToPackage{hyphens}{url}\usepackage[
  hidelinks,
  colorlinks=true,
  allcolors=black,
  pdfstartview=Fit,
  breaklinks=true]{hyperref} 
\usepackage[all]{hypcap} 

\usepackage[capitalise,noabbrev]{cleveref}

\usepackage{pdfcomment}

\newcommand{\eps}{\ensuremath{\varepsilon}\xspace} 

\newcommand{\A}{\ensuremath{\mathcal{A}}\xspace}

\newcommand{\E}{\mathop{\mathbb{E}}}

\newcommand{\Inc}{\ensuremath{\mathsf{M}}}

\newcommand{\Adv}{\ensuremath{\mathsf{Adv}}}
\newcommand{\Advinc}{\Adv^\Inc}

\newcommand{\model}{\ensuremath{A_S}\xspace}
\newcommand{\modeli}{\ensuremath{A_{S^{(i)}}}\xspace}

\usepackage{SCITEPRESS}     

\begin{document}

\title{Privacy in Practice: Private COVID-19 Detection in X-Ray Images\\\Large Extended Version}

\author{\authorname{
  Lucas Lange\sup{1}\orcidAuthor{0000-0002-6745-0845},
  Maja Schneider\sup{1}\orcidAuthor{0000-0001-5936-1415},
  Peter Christen\sup{2}\orcidAuthor{0000-0003-3435-2015} and
  Erhard Rahm\sup{1}\orcidAuthor{0000-0002-2665-1114}}
\affiliation{\sup{1}Leipzig University \& ScaDS.AI Dresden/Leipzig, Leipzig, Germany}
\affiliation{\sup{2}The Australian National University, Canberra, Australia}
\email{\{lange, mschneider, rahm\}@informatik.uni-leipzig.de, peter.christen@anu.edu.au}
}


\abstract{
    Machine learning (ML) can help fight pandemics like COVID-19 by enabling rapid screening of large volumes of images.
    To perform data analysis while maintaining patient privacy, we create ML models that satisfy Differential Privacy (DP).
    Previous works exploring private COVID-19 models are in part based on small datasets, provide weaker or unclear privacy guarantees, and do not investigate practical privacy.
    We suggest improvements to address these open gaps.
    We account for inherent class imbalances and evaluate the utility-privacy trade-off more extensively and over stricter privacy budgets.
    Our evaluation is supported by empirically estimating practical privacy through black-box Membership Inference Attacks (MIAs).
    The introduced DP should help limit leakage threats posed by MIAs, and our practical analysis is the first to test this hypothesis on the COVID-19 classification task.
    Our results indicate that needed privacy levels might differ based on the task-dependent practical threat from MIAs.
    The results further suggest that with increasing DP guarantees, empirical privacy leakage only improves marginally, and DP therefore appears to have a limited impact on practical MIA defense.
    Our findings identify possibilities for better utility-privacy trade-offs, and we believe that empirical attack-specific privacy estimation can play a vital role in tuning for practical privacy.
}

\onecolumn \maketitle \normalsize \setcounter{footnote}{0} \vfill

\section{\uppercase{Introduction}}\label{sec:introduction}
The COVID-19 pandemic pushed health systems worldwide to their limits, showing that rapid detection of infections is vital to prevent uncontrollable spreading of the virus.
Detecting COVID-19 in patients can be achieved using a RT-PCR test\footnote{Reverse Transcription Polymerase Chain Reaction (RT-PCR) testing is broadly used for COVID-19 diagnosis.}.
Although they are more reliable in terms of sensitivity than rapid antigen tests, results can take hours to arrive, and even if displaying negative, the virus could have already left the throat and manifested itself in the lungs, rendering it undetectable for either test~\citep{ALBERT2021472}.

\begin{figure}
  \centering
  \includegraphics[trim=100 0 100 0,clip,width=0.32\linewidth]{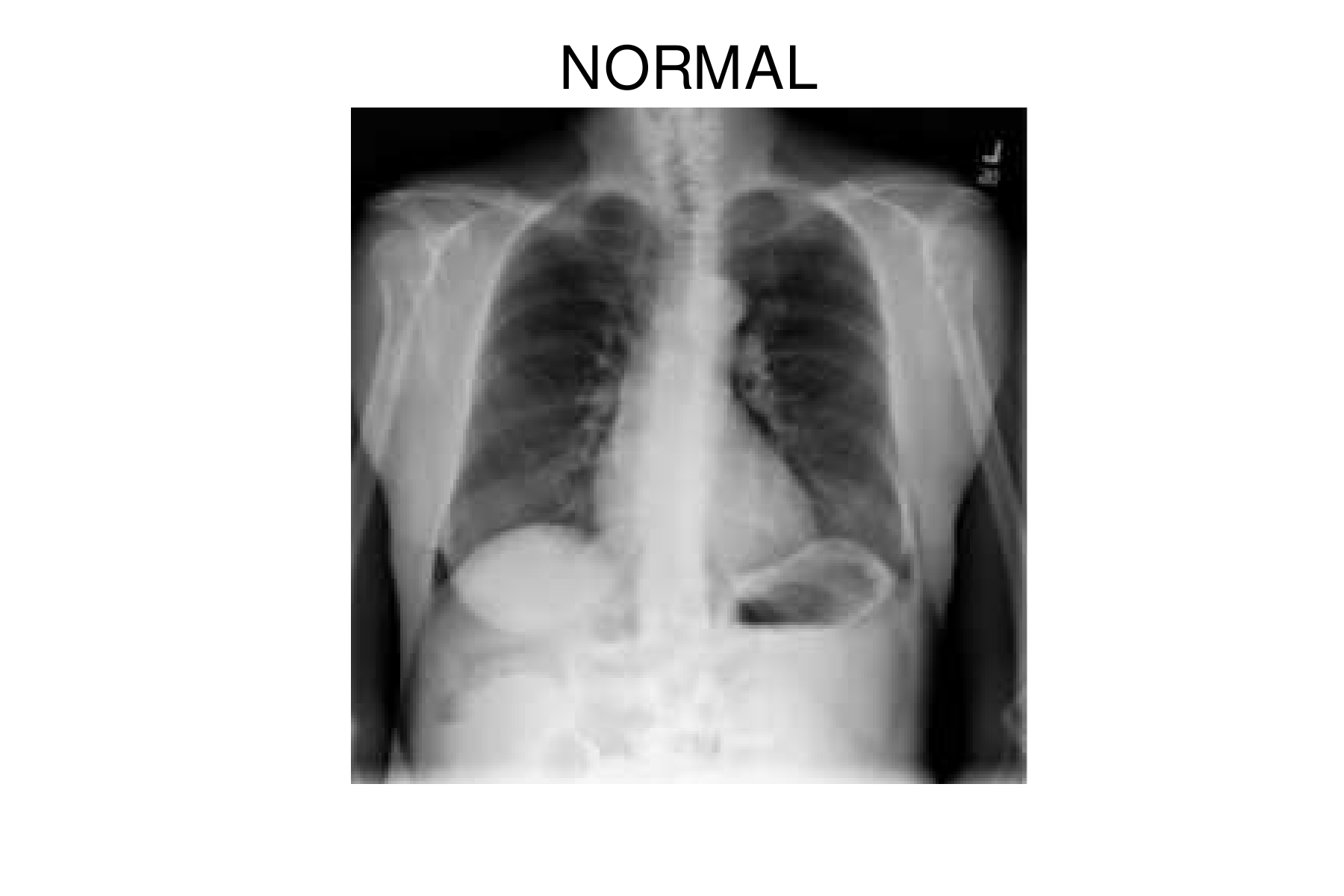}
  \includegraphics[trim=100 0 100 0,clip,width=0.32\linewidth]{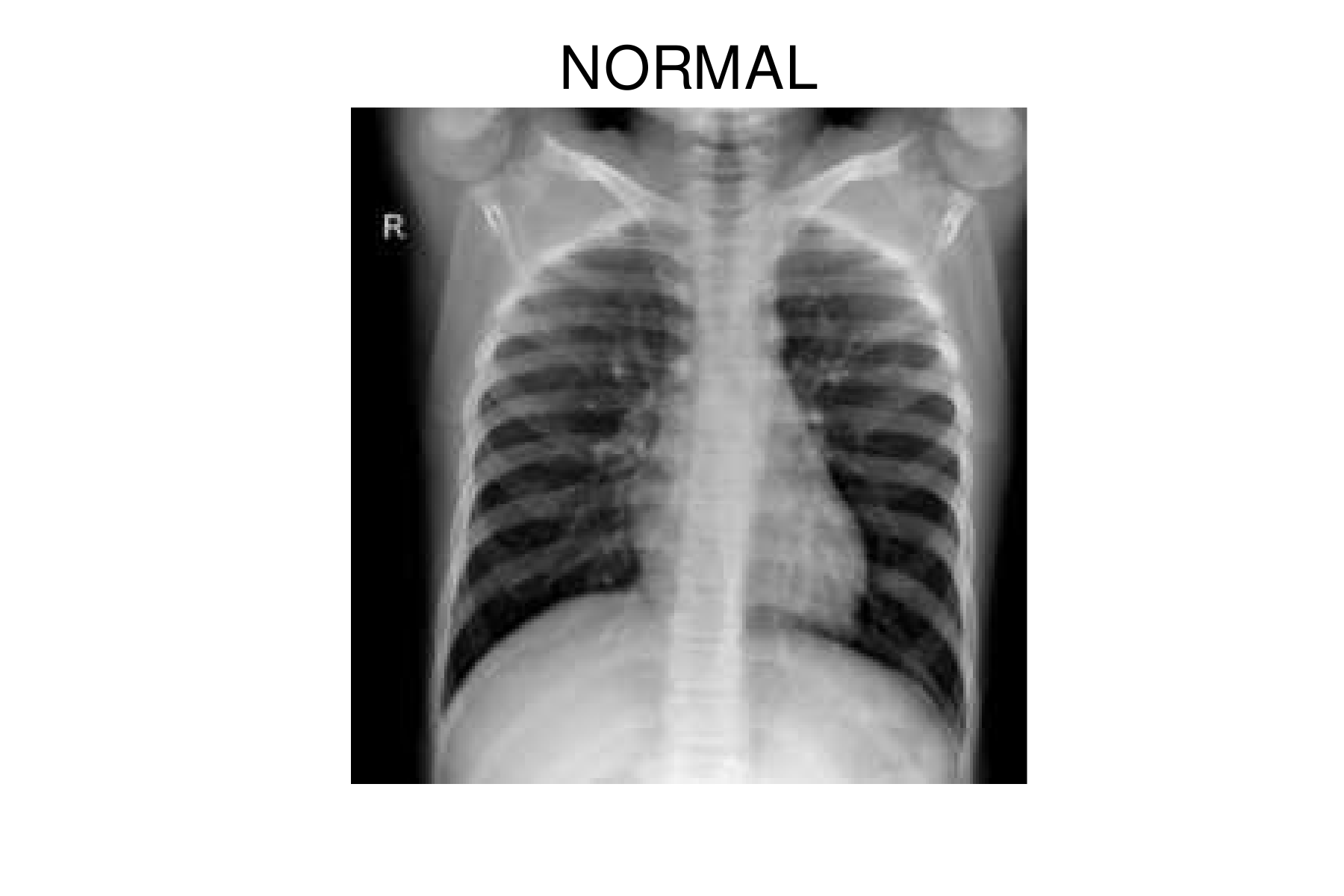}
  \includegraphics[trim=100 0 100 0,clip,width=0.32\linewidth]{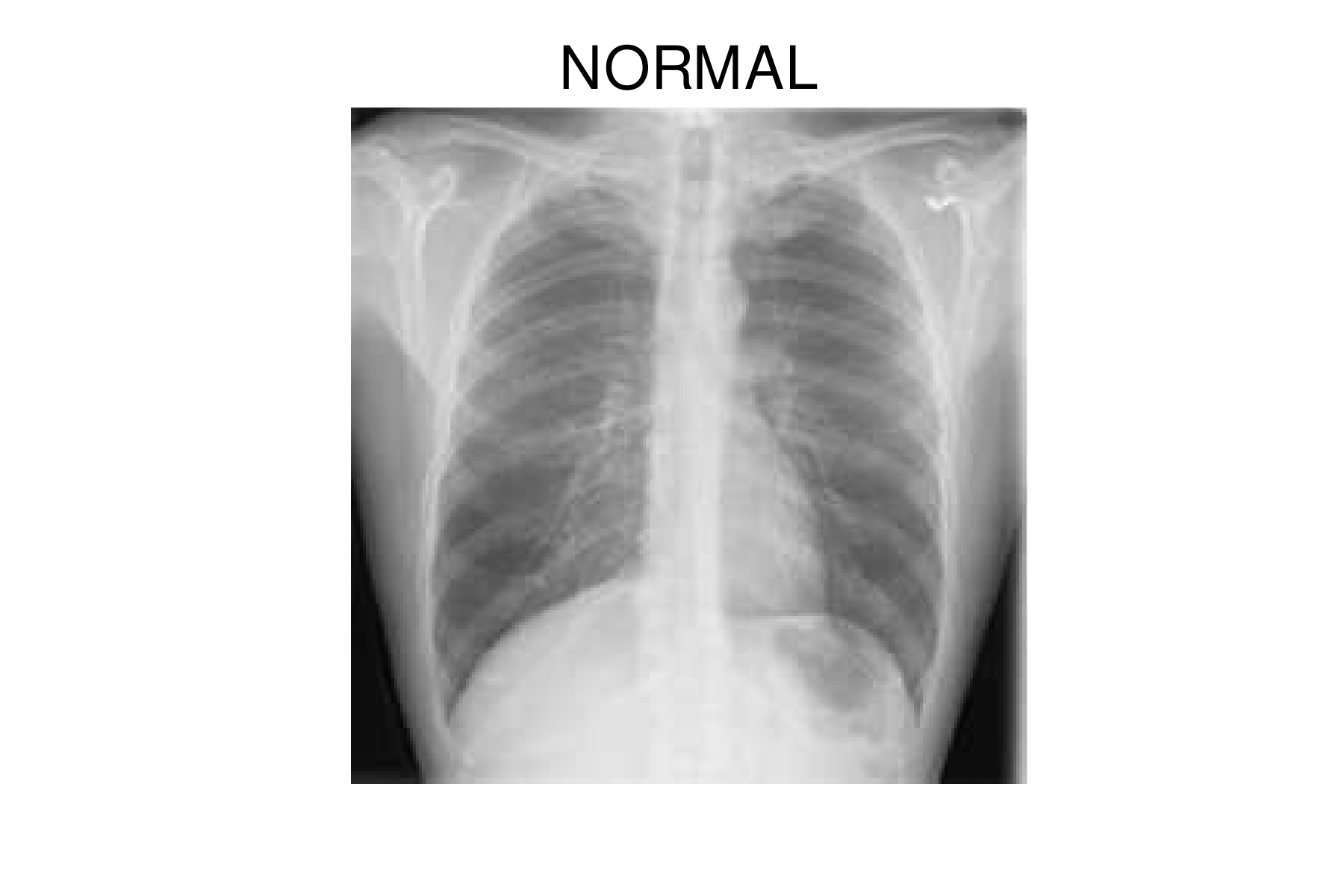}
  \includegraphics[trim=100 0 100 0,clip,width=0.32\linewidth]{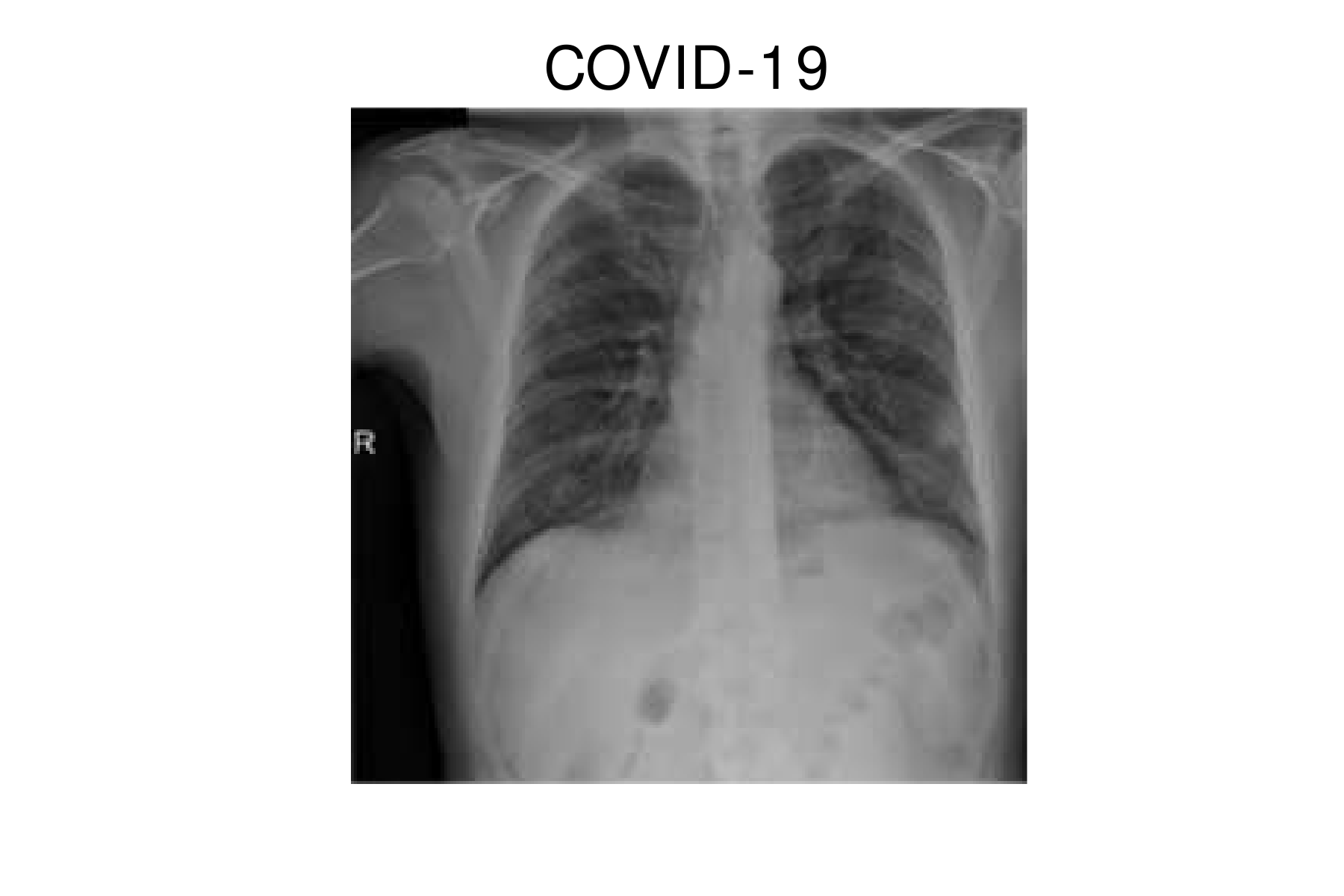}
  \includegraphics[trim=100 0 100 0,clip,width=0.32\linewidth]{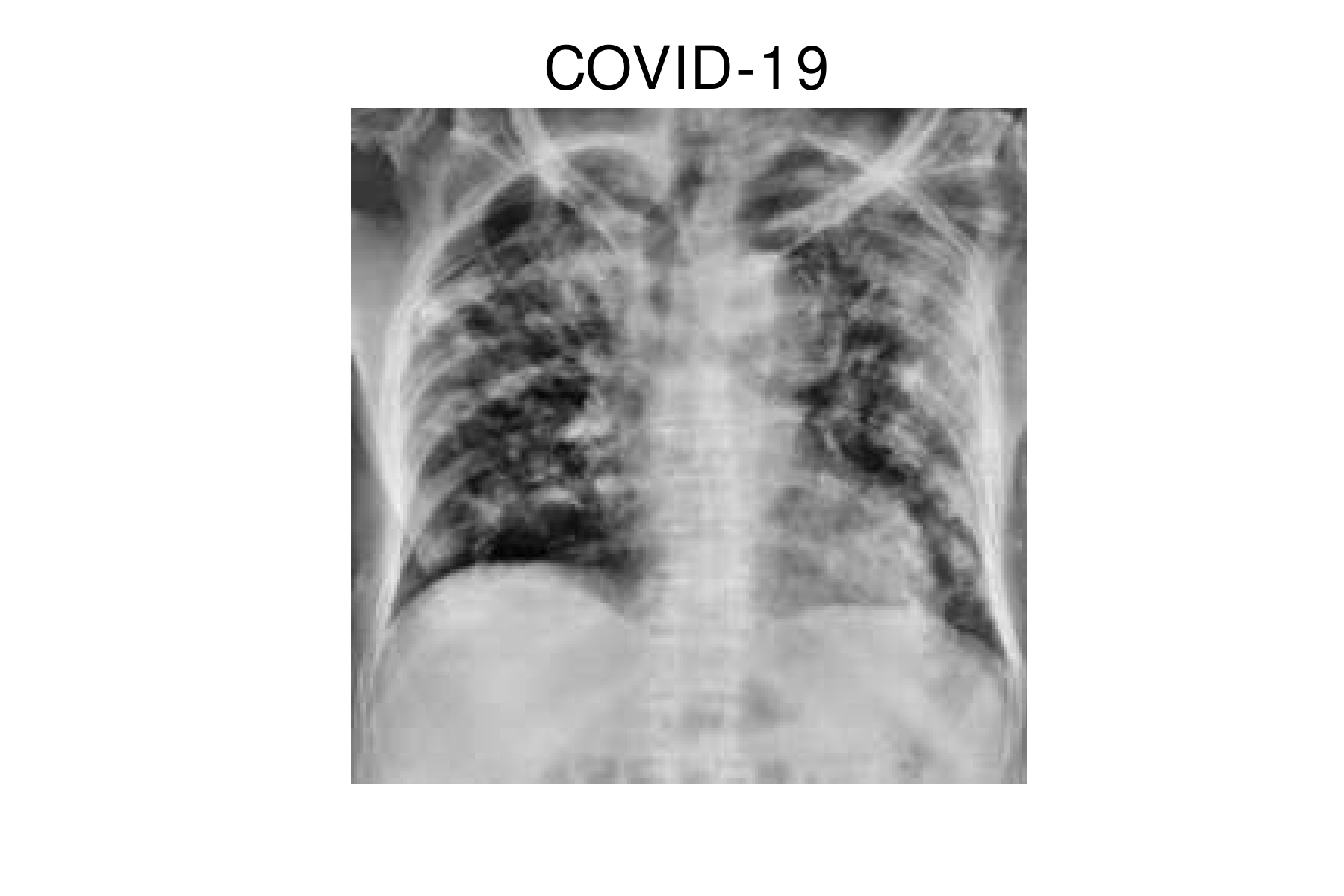}
  \includegraphics[trim=100 0 100 0,clip,width=0.32\linewidth]{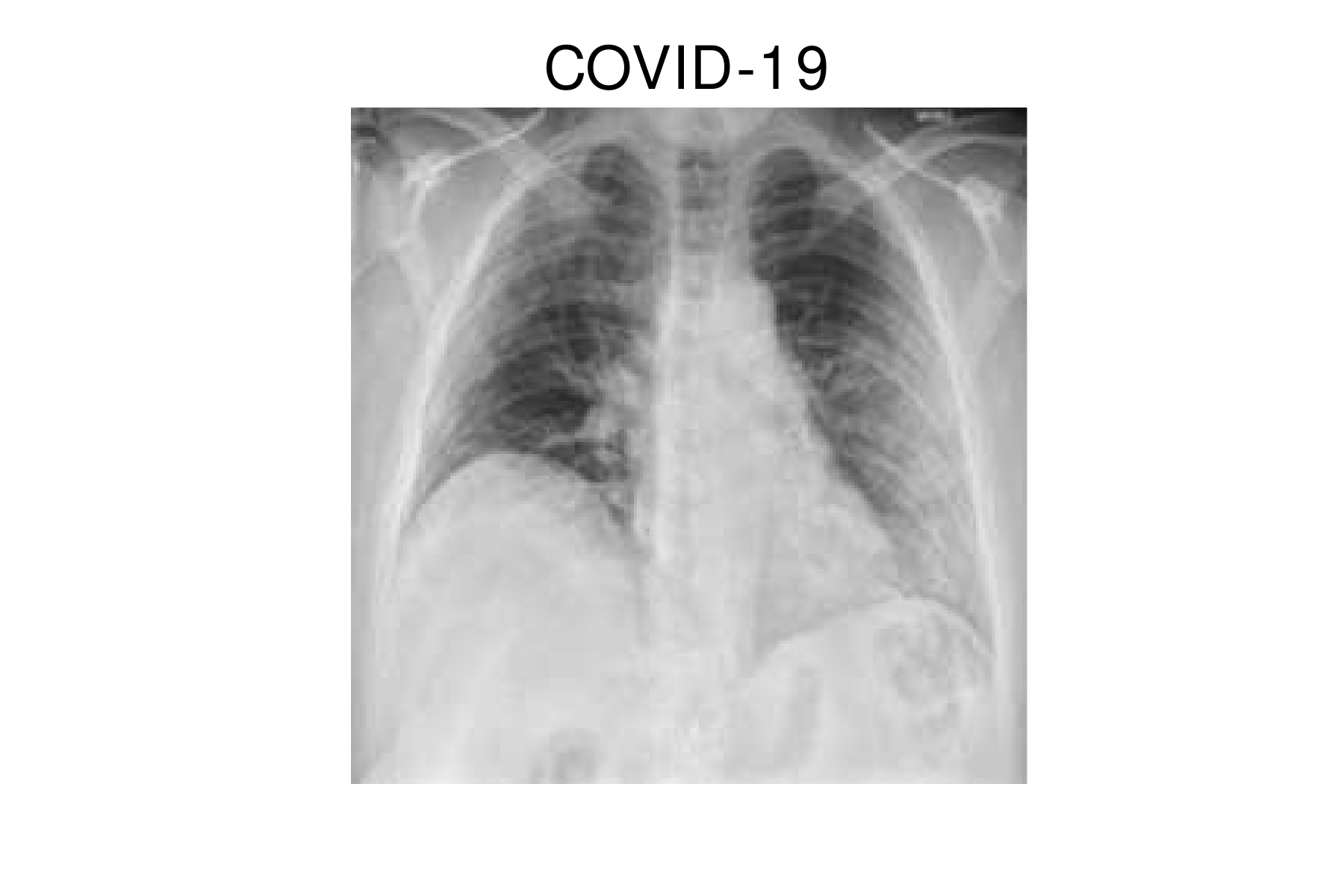}
  \caption{
      Chest X-ray images of different patients extracted from the COVID-19 Radiography Database~\citep{chowdhury2020can,rahman2021exploring}.
      COVID-19 positive scans are characterized by patchy consolidations of the lungs.
  }\label{fig:progression}
\end{figure}

In hospitals, chest X-rays can mitigate these drawbacks by enabling a fast and reliable diagnosis.
\cref{fig:progression} shows chest X-ray scans of healthy (top) and COVID-19 (bottom) patients in direct comparison.
Even though patchy consolidations are recognizable in the COVID-19 scans, such X-rays remain challenging
to interpret.
Specialists, however, are able to identify the severity of a case early on and can take measures without waiting for lab results.

\begin{table*}
  \caption{
    Existing solutions from related work next to our private model at $\eps=1$~\citep{muftuouglu2020differential,zhang2021feddpgan,ho2022dpcovid}.
    The methods refer to training private models.
    For differentiating the tasks, we assign the classes as COVID-19 (C), Normal (N), or Pneumonia (P).
    A performance comparison is difficult due to the different characteristics.
    Baseline shows the best non-private models.
    Our best private result is based on accuracy for comparison to related work.
    We further include two proposed additions for filling open gaps: F1-score and MIA.
  }\label{tab:eval-related}
  \centering\resizebox{\textwidth}{!}{\begin{tabular}{cccc cc ccc}
\toprule
 && \multirow{2}{*}{Number} && \multicolumn{2}{c}{Accuracy in \%} \\
\cmidrule(lr){5-6}
Related Work & Method & of samples & Task & Baseline & Private & \eps & F1? & MIA? \\
\midrule
\multirow{2}{*}{\citet{muftuouglu2020differential}} & \multirow{2}{*}{PATE on EfficientNet-B0} & 139 COVID-19 & binary & \multirow{2}{*}{94.7} & \multirow{2}{*}{71.0} & \multirow{2}{*}{5.98} & \multirow{2}{*}{$\times$} & \multirow{2}{*}{$\times$} \\
  && 234 Normal & C $|$ N \\
\midrule
\multirow{3}{*}{\citet{zhang2021feddpgan}} & ResNet on images from & 350 COVID-19 & \multirow{2}{*}{multi-class} & \multirow{3}{*}{92.9} & \multirow{3}{*}{94.5} & \multirow{3}{*}{?} & \multirow{3}{*}{$\times$} & \multirow{3}{*}{$\times$} \\
  & DP-GAN trained in & 2,000 Normal & \multirow{2}{*}{C $|$ N $|$ P} \\
  & federated learning & 1,250 Pneumonia \\
\midrule
\multirow{3}{*}{\citet{ho2022dpcovid}} & DP-SGD in federated & 3,616 COVID-19 & \multirow{2}{*}{multi-class} & \multirow{3}{*}{95.3} & \multirow{3}{*}{68.7} & \multirow{3}{*}{39.4} & \multirow{3}{*}{$\times$} & \multirow{3}{*}{$\times$} \\
  & learning on custom CNN & 10,192 Normal & \multirow{2}{*}{C $|$ N $|$ P} \\
  & with spatial pyramid pooling & 1,345 Pneumonia \\
\midrule
\multirow{3}{*}{Lange et al.~(ours)} & DP-SGD on ResNet18 & \multirow{2}{*}{3,616 COVID-19} & \multirow{2}{*}{binary} & \multirow{3}{*}{96.8} & \multirow{3}{*}{75.2} & \multirow{3}{*}{1} & \multirow{3}{*}{\checkmark} & \multirow{3}{*}{\checkmark} \\
  & with tanh activation and & \multirow{2}{*}{5,424 Normal} & \multirow{2}{*}{C $|$ N} \\
  & pre-training on Pneunomia \\
\bottomrule
\end{tabular}
}
\end{table*}

Machine Learning (ML) techniques can effectively assist medical professionals in an initial screening by quickly classifying large numbers of images.
However, the amount of data needed for training such classifiers poses problems due to clinical data privacy regulations, which present strict limitations to data sharing between hospitals.
All sensitive patient information must be treated confidentially before, during, and after processing.~\citep{balthazar2018sharing}

To complicate matters, not only the dataset itself but also the models resulting from ML can compromise privacy.
Published models are vulnerable to attacks, including leaking details about their training data~\citep{shokri2017membership}.
Such leaks allow adversaries to potentially deduce sensitive medical facts about individuals in the dataset, for instance by exposing a patient’s genetic markers~\citep{homerResolvingIndividualsContributing2008}.

In the case of COVID-19 detection, an attacker could be able to reveal if a person was infected, which would already violate privacy.
While the specific risk of X-ray-based attacks might be low, such data should be handled with caution, especially since even with anonymization, results can still be linked to other information like related medications.
Furthermore, we cannot rule out an attacker with internal access to images, e.g. doctors utilizing the model in a hospital.

Privacy-Preserving ML (PPML) is a collection of methods for creating trustworthy ML models, enabling, for example, the development of medical applications while maintaining patient privacy.
In this work, we apply PPML that satisfies Differential Privacy (DP)~\citep{dwork2008differential} in training a COVID-19 detection model, thus limiting attacks on the resulting classifier from incurring information leakage.

Our investigation is divided into three successive steps:
(1) First, a non-private baseline is trained to detect COVID-19 versus normal (no findings) in chest X-rays.
(2) The second step then focuses on experiments evaluating ML model architectures and parameters in private training, with the primary objective of finding a feasible utility-privacy trade-off.
(3) Finally, model privacy is empirically assessed by attempting to identify training data through black-box Membership Inference Attacks (MIAs), examining to what extent these models leak private information.

\subsubsection*{Our contributions are:}
\begin{itemize}
    \item 
      We fill open gaps from previous work~\citep{muftuouglu2020differential,zhang2021feddpgan,ho2022dpcovid}, where \cref{tab:eval-related} shows their characteristics in comparison to our approach.
      We address the class imbalances and analyze the utility-privacy trade-off more extensively by evaluating multiple and stricter privacy budgets.
      We further investigate practical privacy by empirically estimating privacy leakage through black-box MIAs.
      These gaps and our improvements are addressed throughout the following sections.
    \item
      We are the first to evaluate if DP helps narrow down MIAs on the COVID-19 detection task.
      We additionally re-examine this hypothesis on a common benchmarking dataset to reveal connections between the two datasets.
      Our results point towards identifying the benefits from DP in defending against MIAs as task-dependent and plateauing.
      We are able to gain better utility-privacy trade-offs at no practical cost.
      These results thus strengthen the belief that empirical privacy analysis can be a vital tool in supporting attack- and task-specific tuning for privacy.
\end{itemize}

\newpage \noindent
The following \cref{sec:background} provides an overview of essential concepts.
We then contextualize our work by examining the existing literature in \cref{sec:related}, and we present our selected solutions to address open research gaps in \cref{sec:methods}.
\cref{sec:experiments} lays out our experimental setup, with the results presented in \cref{sec:results} and their discussion in \cref{sec:discussion}.
In closing, \cref{sec:conclusion} provides conclusive thoughts and adds an outlook to possible future work.

\section{\uppercase{Background}}\label{sec:background}
This section establishes a basic understanding of the relevant concepts and algorithms used in this work.

\subsection{Differential Privacy}\label{sec:back-dp}
DP is the quasi gold standard in private data analysis, which offers a guarantee that the removal or addition of a single dataset record does not (substantially) affect the outcome of any analysis~\citep{dwork2008differential}.
Thus, an attacker is incapable of differentiating from which of two neighboring datasets a given result originates and has to resolve to a random guess---i.e., a coin flip.
DP's provided guarantee is measured by giving a theoretical upper bound of privacy loss, represented as the privacy budget $\eps$.
The metric is accompanied by the probability of privacy being broken by accidental information leakage, which is denoted as $\delta$ and depends on the dataset size.

Formally, an algorithm A training on a set S is called ($\eps$,$\delta$)-differentially-private, if for all datasets D and D' that differ by exactly one record:
\begin{equation}\label{eq:dp}
  Pr[A(D) \in S] \leq e^{\eps} Pr[A(D') \in S] + \delta
\end{equation}
Meaningful privacy guarantees in ML should fulfill $\eps \leq 1$ and $\delta \ll 1/n$, where $n$ is the number of training samples~\citep{nasr2021adversary,carlini2019secret}.
The notation $\eps=\infty$ indicates that no DP criteria are met.

The design of DP algorithms is based on one of DP's fundamental properties: composability~\citep{dwork2014algorithmic}.
It states that if all the components of a mechanism are differentially-private, then so is their composition.
Another essential attribute of DP is its post-processing immunity, implying DP is preserved by all further processing.
Therefore, in terms of achieved privacy, it does not matter whether a ML model uses an already DP conform dataset or applies DP while training~\citep{dwork2014algorithmic}.

\subsection{Differentially-Private Stochastic Gradient Descent}\label{sec:back-dpsgd}
The Differentially-Private Stochastic Gradient Descent (DP-SGD) algorithm introduced by \citet{abadi2016deep} takes widely used SGD and applies a gradient perturbation strategy.
Gradient perturbation adds enough noise to the intermediate gradients to obfuscate the largest value, since that original sample inhibits the highest risk of exposure.
To generally bound the possible influence of individual samples while training, DP-SGD clips gradient values to a predefined maximum Euclidean norm before adding noise.
The noisy gradients are then used to update the parameters as usual.
The total noise added through the algorithm is composed over all training iterations using an accounting mechanism and determines the resulting privacy budget.

\subsection{Membership Inference Attacks}\label{sec:back-membership}
In black-box MIAs, an attacker feeds data samples to a target model and thereby tries to figure out each sample's membership or absence in the model's training set based solely on the returned confidence values.
This technique takes advantage of the differences in predictions made on data used for training versus unseen data, where the former is expected to output higher confidence values due to memorization~\citep{carlini2019secret}.
As proposed by \citet{shokri2017membership}, such attacks can utilize multiple shadow models specifically mimicking a target model's predictions, to train an attack model able to elicit the desired membership information.
\citet{salem2019ml} relaxed the need for shadow models, by finding that simply using the original model's predictions on given samples can be sufficient to deduce their membership.

By revealing the membership of an individual's record in the dataset, an adversary might in turn disclose sensitive information on them.
However, a decisive prerequisite is some form of existing prior knowledge of the target data, because the MIA only tests samples available to the attacker~\citep{shokri2017membership}.
Thus, the attacker needs to possess a set of individual samples to uncover their membership.

\section{\uppercase{Related Work}}\label{sec:related}
In the following, we first describe gaps left open by related work in \cref{sec:related-detection}.
We then show mitigation strategies for MIAs and methods of practical privacy analysis in \cref{sec:back-repelling,sec:related-practical}, respectively.

\subsection{Private COVID-19 X-Ray Detection}\label{sec:related-detection}
In \cref{tab:eval-related}, existing works on private COVID-19 detection from X-rays are summarized and compared to our approach.
There are multiple factors that impede a fair comparison, which mainly lie in the differences in datasets, tasks, and privacy guarantees (\eps).
In this section, we show open gaps and then give elaborations in \cref{sec:methods} on how we address them.


\textit{Datasets.}
    A problem regarding~\citep{muftuouglu2020differential} is that their results are based on only a small dataset of 139 COVID-19 scans.
    The COVID-19 Radiography Database used by~\citep{ho2022dpcovid} and us, provides a better basis in terms of dataset size.
    However, the class imbalances result in a rather skewed data basis, which is left unaddressed but could influence MIA threat~\citep{jayaramanRevisitingMembershipInference2021}.
    With the FedDPGAN approach, \citet{zhang2021feddpgan} try to enlarge and balance their small dataset using synthetic images, but the quality of the generated distribution is left unanswered.
    This is particularly problematic because GANs trained on imbalanced input data tend to produce data with similarly disparate impacts~\citep{ganev2022robin}.
    As a general problem with skewness, the mentioned works solely assess performance using accuracy, although this metric is known to undervalue false negatives for minority classes and could favor classifiers that are actually worse in detecting the COVID-19 minority class~\citep{bekkar2013evaluation}.

\textit{Privacy budgets.}
    The used \eps-values of 5.98 and 39.4 by \citet{muftuouglu2020differential} and \citet{ho2022dpcovid} respectively, are significantly weaker than the privacy budget of $\eps \leq 1$, which is commonly assumed to provide strong privacy~\citep{nasr2021adversary,carlini2019secret}.
    Furthermore, the results by \citet{zhang2021feddpgan} lack comparability, since they do not provide their privacy budget.
    Using their parameters and noise in a standard DP-SGD analysis results in $\eps>5*10^{13}$ for a client after 500 rounds\footnote{They do not state their exact number of rounds but their graphs show 500 rounds.} of federated training.
    Even with their most private setting they still accumulate $\eps=19.6$.
    Thus, no model adheres to $\eps \leq 1$ and they instead only offer weaker or unclear guarantees.

\textit{Practical privacy.}
    Regarding practical privacy, prior work does not include actual attack scenarios.
    It is therefore left open to what extent the provided models and \eps-guarantees retain patient privacy against real adversaries.
    Such analysis helps in assessing the defense capabilities provided by the achieved privacy budgets and could reveal room for tuning them.

\subsection{Repelling MIAs}\label{sec:back-repelling}
Related work suggests multiple strategies for reducing MIA threats.
\citet{shokri2017membership} show that limiting the model outputs to only class labels instead of explicit confidence values can be an effective remedy.
However, in medical tasks such as COVID-19 detection, where the use case is to help medical professionals in diagnosing a disease, the confidence value is an integral part that indicates how likely a patient is affected.
\citet{shokri2017membership} also find that model architecture can contribute to MIA defense and \citet{salem2019ml} demonstrate that even the training process can hinder MIAs through e.g. model stacking.

DP should limit and oppose the success of MIAs by design, with \citet{jayaraman2019evaluating} supplying the corresponding reasoning:
\enquote{[DP], by definition, aims to obfuscate the presence or absence of a record in the data set.
  On the other hand, [MIAs] aim to identify the presence or absence of a record \textelp{}.}
\citet{rahman2018membership} test this hypothesis by evaluating MIAs on different privacy levels.
They find their model's MIA resistance to gradually increase when lowering the allowed privacy budget and explain it with less overfitting when adding more noise.
\citet{yeom2018privacy} prove that overfitting in ML models is sufficient to enable MIAs, but at the same time show that overfitting is not a necessary criterion, and stable models can still be vulnerable.

\subsection{Practical Privacy Analysis}\label{sec:related-practical}
Multiple works examined the possibilities of estimating the practical privacy for ML models by performing an empirical study through attacks, e.g. MIAs.
\citet{jagielski2020auditing} and \citet{nasr2021adversary} conclude that the assumed theoretical upper bound privacy loss for DP, given in the privacy budget \eps, gives a tight worst-case analysis on attack proneness and thereby limits MIA success.
However, in many cases actual attacks extract significantly less information than assumed by the theoretical bound,  which is also supported by \citet{malek2021antipodes} and \citet{jayaraman2019evaluating}.
This discrepancy could possibly enable better utility-privacy trade-offs, but \citet{jayaraman2019evaluating} warn that privacy always comes at a cost and reducing privacy could ultimately promote information leakage.
\citet{malek2021antipodes} propose that a realistic lower bound on the amount of revealed information by a model can be determined by \enquote{[considering] the most powerful attacker pursuing the least challenging goal} and that in the case of standard DP, such would be an attacker powerful enough to successfully perform membership inference.
Therefore, by attacking our models with MIAs, we can empirically estimate practical privacy leakage which might differ substantially from the upper leakage bound derived from DP theory.

\section{\uppercase{Methods}}\label{sec:methods}
As seen in \cref{sec:related-detection} and \cref{tab:eval-related}, related work on COVID-19 detection lacks comparability and leaves open research gaps.
We therefore do not solely focus on enhancing the performance of former solutions but rather suggest improvements by filling existing gaps, ultimately proposing the following improvements.

\textit{Datasets.}
    Since the dataset used by us and \citet{ho2022dpcovid} provides a good amount of COVID-19 samples, we instead aim for better handling of the problems arising from the skewed nature of the class distribution.
    In a first effort, we employ random undersampling and class weights to elevate the underrepresented COVID-19 class, both in database construction and in training, respectively.
    Furthermore, since accuracy is not representative in cases of skewed data, we improve the evaluation by using the more balanced F1-score metric~\citep{bekkar2013evaluation}.
    
\textit{Privacy budgets.}
    We investigate the utility-privacy trade-off by evaluating multiple and stricter privacy budgets of $\eps=[\infty, 10, 1, 0.1]$.
    To find the best private model and extend the pool of evaluated methods, we propose untested architectural experiments relevant to private DP-SGD training in \cref{sec:archis}.
    
\textit{Practical privacy.}
    As seen in the works discussed in \cref{sec:related-practical}, we investigate the practical implications of DP regarding the defense against black-box MIAs by undertaking an empirical analysis through actual attacks, and therefore give a more realistic lower bound to the resulting privacy leakage~\citep{jagielski2020auditing,malek2021antipodes}.
    We thereby provide the first attack results in the field of private COVID-19 detection and evaluate possible room for tuning the utility-privacy trade-off.
    An additional evaluation regarding the privacy leakage of our models on the MNIST database enables us to formulate takeways regarding similarities and disproportions regarding the attack-specific privacy on both datasets.
    Evaluating another dataset is a first step towards generalization and MNIST is particularly interesting because related works~\citep{rahman2018membership,nasr2021adversary} previously investigated the connection between DP and MIA on this task.

\section{\uppercase{Experimental Setup}}\label{sec:experiments}
\begin{table}
  \caption{
    Summary of the experiment parameters.
    Each combination from left to right constitutes a possible setup (resulting in $2*2*2*3*4=96$ setups).
  }\label{tab:experiments}
  \centering\resizebox{0.47\textwidth}{!}{\begin{tabular}{ccccc}
\toprule
Dataset & Architecture & Activation & Pre-training & \eps \\
\midrule
\multirow{3}{*}{COVID-19} & \multirow{3}{*}{ResNet18} & \multirow{3}{*}{ReLU} & \multirow{2}{*}{None/Standard} & $\infty$ \\
\multirow{3}{*}{MNIST} & \multirow{3}{*}{ResNet50} & \multirow{3}{*}{tanh} & \multirow{2}{*}{ImageNet} & 10 \\
  &&& \multirow{2}{*}{Pneumonia} & 1 \\
  &&&& 0.1 \\
\bottomrule
\end{tabular}}
\end{table}

In this section we provide details on the setups used in our experiments, which are summarized in \cref{tab:experiments}.
Reference code is available from our repository\footnote{\url{https://github.com/luckyos-code/mia-covid}}.

\subsection{Environment}
The implementation uses Python with Keras\footnote{\url{https://keras.io/}} and Tensorflow\footnote{\url{https://www.tensorflow.org/}}.
Further, we employ the modules for DP-SGD training and MIAs from the Tensorflow Privacy library\footnote{\url{https://github.com/tensorflow/privacy}}.
To enable consistent and reproducible results, random seeds are set to a fixed value, which is 42 in our case.
Hardware-wise our machines are equipped with 64GB of RAM and an NVIDIA Tesla V100-PCIE-32GB GPU.

\subsection{Datasets}\label{sec:datasets}
For a comprehensible dataset creation, we provide details on the different public datasets we used.

\begin{itemize}
\item The \textit{COVID-19 Radiography Database\footnote{\url{https://www.kaggle.com/tawsifurrahman/covid19-radiography-database}}}~\citep{chowdhury2020can,rahman2021exploring} is the most comprehensive collection of COVID-19 chest X-ray images, stemming from different databases around the web.
In total, this image collection offers chest X-rays of 3,616 COVID-19 positive, 10,192 Normal, and 1,345 Pneumonia cases.
For our binary task, we omit the pneumonia samples and employ undersampling to directly reduce class imbalances.
Dataset construction takes all COVID-19 scans but only 1.5$\times$ the amount for Normal (5,424), resulting in 9,040 images total.
When testing hyperparameters, this ratio showed to elevate performance (F1) and reduce privacy risk due to less overfitting~\citep{rahman2018membership}.

\item The \textit{Chest X-Ray Images (Pneumonia)}\footnote{\url{https://www.kaggle.com/paultimothymooney/chest-xray-pneumonia}}~\citep{kermany2018identifying} offers X-ray images divided into two classes with 1,583 Normal and 4,273 Pneumonia samples.
Here, we again apply undersampling to achieve similar class ratios and take all Normal scans but just 1.5$\times$ the amount for Pneumonia (2,374).
This pneumonia dataset is also part of the COVID-19 Radiography Database, constituting 13\% (1,341) of its Normal class images.
To fix this issue and enable its use as a public dataset for our private transfer learning approach without compromising privacy, we exclude duplicates when sampling images for the COVID-19 task.

\item The \textit{ImageNet}~\citep{deng2009imagenet} is a vast collection counting 14 million images and covering 20,000 categories from general (mammal) to specific (husky).
Non-private models benefit from using this massive dataset for pre-training, introducing many differentiating concepts to a neural network before training on the target data.

\item With their \textit{MNIST Database}~\citep{lecun1998gradient}, LeCun et al. offer a large image collection of handwritten digits.
The database provides 60,000 images for training and 10,000 for testing.
Even though MNIST does not contain COVID-19 related images, it is a commonly used benchmark in image classification and PPML, making it a perfect candidate for comparing results. 
\end{itemize}

Our experiments are mainly focusing on the task-relevant COVID-19 Radiography Database as the evaluation basis.
Additionally, models are then evaluated on the MNIST database.
Pre-training uses the respective public ImageNet and Pneumonia datasets.

\subsection{Pre-processing}\label{sec:preprocessing}
To build our final splits for model training, we employ necessary sampling and pre-processing steps.
Both X-ray datasets, for COVID-19 and pneumonia, use a train-validation-test split of 80\% training, 5\% validation, and 15\% test set.
All datasets are handled with three color channels.
We therefore convert the MNIST grey-scale images into the RGB space, as this is vital to allow the models pre-trained on color images to still work with the input data.
X-ray images are downscaled to 224x224 pixels, while MNIST images keep their size of 28x28.
Furthermore, each data point undergoes an image normalization step using the factor of x=1/255.

To combat overfitting, training sets are shuffled and training images from the X-ray datasets are subjected to data augmentation~\citep{shorten2019survey}.
An problem already improved from the X-ray datasets' undersampling is their imbalance regarding Sickness (COVID-19, Pneumonia) and Normal class frequencies.
To further improve on this issue, we apply class weights during training to help artificially balance each sample's impact.
This process yields class weights of 0.83 for the Normal and 1.25 for the Sickness classes.

\subsection{Architectural Experiments}\label{sec:archis}
In the following, we describe our different architectural choices we used for experiments.

\subsubsection{Model Size in DP-SGD}
Non-private and private classification perform differently depending on the underlying model architecture~\citep{papernot2021tempered}.
Performance is greatly dependent on model size, i.e., the number of layers or parameters in a model, with non-private training typically benefiting from using bigger models.
In contrast, using DP-SGD, the same models suffer from accuracy loss when increasing in size.
Taking these findings into account, in our experiments we used two differently-sized architectures.

\subsubsection{ResNet Models}
The model family of Residual Networks~\citep{he2016deep}, or ResNets, provides well-scaling deep CNN architectures thanks to its residual connections.
We utilize two different versions as our basic model architectures for the experiments.
For one, the ResNet50 version boasting 50 layers, which is a stretched compromise between size and computing needs, as well as the ResNet18, which is mostly identical except for the lower layer count of 18 layers.
When introducing architectural experiments, we change each basic model accordingly, creating different variants.

\subsubsection{tanh Activation}
A disruptive discovery in DP-SGD research was made by \citet{papernot2021tempered}.
In their work, they determined that replacing the de facto standard ReLU activation function with the tanh function in model layers improves performance in DP-SGD.
To achieve this boost, they utilize the fact that the tanh activation generally results in smaller gradients than the ReLU function, which in turn reduces the information loss from gradient clipping.

\subsubsection{Pre-training}
An important consideration in private training is the use of public pre-training datasets.
This is due to the advantage that public datasets do not require the same noisy training mechanisms as private datasets.
On the basis of the resulting pre-tuned weights, the pre-trained model is then fine-tuned to the private target data for the actual task.

A commonly applied strategy to improve performance for non-private classification relies on pre-training using the extensive ImageNet collection.
As another method, \citet{abadi2016deep} state that DP-SGD models can further profit from pre-training in a domain closely related to the target task.
While ImageNet resembles a general choice for image-based tasks, pre-training for pneumonia detection is closer to our COVID-19 task due to the similarity in symptoms~\citep{speranskaya2020radiological,lange2022privacy}.

The pre-training on the Pneumonia dataset is performed using the same settings as on the COVID-19 set, while the ImageNet variants are provided by a library for Keras models\footnote{\url{https://github.com/qubvel/classification_models}}.
For our tanh variants we take the take the pre-trained ReLU models and change the activation function in each trainable layer before training on our target datasets.

\subsection{Privacy Experiments}
In this section, we elaborate on the used settings and hyperparameters for evaluating privacy.

\subsubsection{Private Training Settings}
All model variants are first trained non-privately using Adam~\citep{Kingma2015AdamAM} optimization at $\eps=\infty$ to form a baseline.
We employ batch sizes of 32 and train for 20 epochs using a learning rate of $\alpha=1\mathrm{e}{-3}$, which decays down to a minimum of $\alpha=1\mathrm{e}{-6}$ on plateaus.
Afterwards, we apply DP-SGD (or here DP-Adam) training to all models, training a private candidate for each $\eps$-guarantee.
The DP-SGD algorithm is applied by changing the optimizer and handing in the necessary privacy parameters, like our employed clipping norm of 1.0.
DP-SGD training for COVID-19 uses ResNet50 and ResNet18 variants with batch sizes of 8 and 16, instead of 32 respectively.
We aim at privacy budgets of $\eps \leq 1$, since such values present strong privacy guarantees~\citep{nasr2021adversary,carlini2019secret}.
We also evaluate budgets neighboring this setting by an order of magnitude, to gain further insights into the performance and estimated privacy on different DP levels.
Due to the dataset size, the DP analysis uses $\delta=1\mathrm{e}{-4}$ for COVID-19 ($n=9,040$) and $\delta=1\mathrm{e}{-5}$ for MNIST ($n=60,000$).

\subsubsection{MIA Settings}
For selecting the most potent MIA each run, we try four different attack types based on logistic regression, multi-layered perceptron, k-nearest neighbors, and threshold.
These attacks found in the Tensorflow Privacy library are an implementation of the single shadow model black-box attack proposed by \citet{salem2019ml}, that directly relies on target model predictions instead of training several shadow models.
Given a target model, MIAs utilize two types of data for training: (1) the original training data to be inferred and (2) unseen but similar data to differentiate non-training data.
In our case, we want to fully empower the attacker for estimating the practical worst-case in an optimal black-box setting~\citep{malek2021antipodes}.
We satisfy this condition by giving access to the full training and test sets with their corresponding labels, thus, handing the attacker the largest input regarding (1) and the most similar input regarding (2).

\subsubsection{Measuring Privacy Leakage}\label{sec:metric}
Like \citet{jayaraman2019evaluating}, our used metric for measuring privacy leakage through MIAs is the attacker's membership advantage as introduced by \citet{yeom2018privacy}.
The adversarial game is based on an attacker's capabilities in differentiating the membership of a sample that is chosen uniformly at random to originate from the training set or not.
The resulting difference in True Positive Rate (TPR) and False Positive Rate (FPR) is then given as the attacker's membership advantage:
\begin{equation}
    \Advinc = TPR - FPR
\end{equation}

\citet{yeom2018privacy} show that if a learning algorithm satisfies \eps-DP, then the adversary’s membership advantage is bounded by $\Advinc \leq e^{\eps}-1$ in their attack scenario.
Transferring the theorem to ($\eps$,$\delta$)-DP given by \cref{eq:dp}, the upper bound can be derived as:
\begin{equation}\label{eq:dpbound}
    \Advinc \leq e^{\eps}-1+\delta
\end{equation}
Because the theoretical assumption relies on Gaussian distributed training errors and a balanced prior data distribution probability, it might not provide reliable bounds given our differing practical scenario.

Since individual MIA results are subject to variability, they need to be experimentally stabilized.
Like \citet{malek2021antipodes}, we achieve this by running 100 entire MIAs and calculating the corresponding 95\% Confidence Interval (CI) for the obtained results.

\begin{table*}[t]
  \caption{
    Experimental results on the COVID-19 dataset.
    The Standard, ImageNet and Pneunomia models rely on the ReLU activation function, which is then changed to tanh in the respective counterparts.
    Model variants are evaluated across multiple DP budgets $\eps$, where $\eps=\infty$ translates to non-private training.
    They are matched by accuracy and F1-score in \%, as well as empirical privacy leakage from MIAs, measured by the membership advantage ($\Advinc$) and given as a 95\% CI over 100 attacks.
    If training resulted in an F1-score of 0.0, no feasible model was derived, making accuracy and attacks obsolete (NA).
  }\label{tab:eval-covid}
  \centering\resizebox{\textwidth}{!}{\begin{tabular}{c ccc ccc ccc ccc}
\toprule
 & \multicolumn{3}{c}{$\eps = \infty$} & \multicolumn{3}{c}{$\eps = 10$} & \multicolumn{3}{c}{$\eps = 1$} & \multicolumn{3}{c}{$\eps = 0.1$} \\
\cmidrule(lr){2-4} \cmidrule(lr){5-7} \cmidrule(lr){8-10} \cmidrule(lr){11-13}
Variant & \%-Acc. & \%-F1 & $\Advinc$ & \%-Acc. & \%-F1 & $\Advinc$ & \%-Acc. & \%-F1 & $\Advinc$ & \%-Acc. & \%-F1 & $\Advinc$ \\
\midrule
\multicolumn{13}{c}{ResNet18} \\
\midrule
 Standard &
  91.4 & 89.5 & 0.22--0.24 &
  71.2 & 57.8 & 0.22--0.24 &
  NA & 0.0 & NA &
  NA & 0.0 & NA \\
 ImageNet &
  \textbf{96.8} & \textbf{95.9} & 0.25--0.27 &
  \textbf{85.5} & \textbf{79.4} & 0.25--0.27 &
  NA & 0.0 & NA &
  NA & 0.0 & NA \\
 Pneumonia &
  92.2 & 89.8 & 0.22--0.24 &
  71.5 & 57.2 & 0.23--0.25 &
  70.5 & 54.3 & 0.22--0.24 &
  71.3 & 61.4 & 0.22--0.23 \\
 tanh-Standard &
  85.1 & 82.8 & \textbf{0.21--0.23} &
  71.8 & 67.9 & 0.21--0.23 &
  71.5 & 62.5 & 0.20--0.22 &
  68.0 & 63.0 & 0.20--0.22 \\
 tanh-ImageNet &
  91.4 & 89.8 & 0.22--0.24 &
  57.5 & 65.2 & \textbf{0.19--0.21} &
  44.5 & 58.9 & 0.20--0.22 &
  50.8 & 61.0 & \textbf{0.19--0.21} \\
 tanh-Pneumonia &
  79.9 & 78.6 & \textbf{0.21--0.23} &
  73.9 & 73.1 & 0.21--0.24 &
  \textbf{75.2} & 70.5 & 0.22--0.24 &
  72.9 & 65.8 & 0.21--0.22 \\
\midrule
\multicolumn{13}{c}{ResNet50} \\
\midrule
 Standard &
  91.6 & 89.3 & 0.25--0.27 &
  NA & 0.0 & NA &
  NA & 0.0 & NA &
  NA & 0.0 & NA \\
 ImageNet &
  95.6 & 94.4 & 0.25--0.26 &
  NA & 0.0 & NA &
  NA & 0.0 & NA &
  NA & 0.0 & NA \\
 Pneumonia &
  91.4 & 89.6 & 0.24--0.26 &
  NA & 0.0 & NA &
  NA & 0.0 & NA &
  NA & 0.0 & NA \\
 tanh-Standard &
  78.8 & 78.0 & \textbf{0.21--0.23} &
  72.3 & 63.4 & 0.22--0.23 &
  70.4 & 62.9 & \textbf{0.19--0.21} &
  68.6 & 62.0 & \textbf{0.19--0.21} \\
 tanh-ImageNet &
  88.8 & 84.9 & 0.23--0.25 &
  47.9 & 60.1 & \textbf{0.19--0.21} &
  46.0 & 59.3 & \textbf{0.19--0.21} &
  50.8 & 60.7 & \textbf{0.19--0.21} \\
 tanh-Pneumonia &
  81.3 & 80.1 & \textbf{0.22--0.23} &
  72.0 & 72.7 & 0.21--0.23 &
  72.0 & \textbf{72.5} & 0.21--0.23 &
  \textbf{73.0} & \textbf{69.4} & 0.21--0.23 \\
\bottomrule
\end{tabular}}
\end{table*}

\begin{table*}
  \caption{
    Experimental results on the MNIST database.
    See \cref{tab:eval-covid} caption for details.
    F1-score is given as the macro average over the 10 classes.
    Pneumonia pre-trained models are omitted from the evaluation, since the tasks are not closely related.
  }\label{tab:eval-mnist}
  \centering\resizebox{\textwidth}{!}{\begin{tabular}{c ccc ccc ccc ccc}
\toprule
 & \multicolumn{3}{c}{$\eps = \infty$} & \multicolumn{3}{c}{$\eps = 10$} & \multicolumn{3}{c}{$\eps = 1$} & \multicolumn{3}{c}{$\eps = 0.1$} \\
\cmidrule(lr){2-4} \cmidrule(lr){5-7} \cmidrule(lr){8-10} \cmidrule(lr){11-13}
Variant & \%-Acc. & \%-F1 & $\Advinc$ & \%-Acc. & \%-F1 & $\Advinc$ & \%-Acc. & \%-F1 & $\Advinc$ & \%-Acc. & \%-F1 & $\Advinc$ \\
\midrule
\multicolumn{13}{c}{ResNet18} \\
\midrule
 Standard &
  \textbf{99.5} & \textbf{99.5} & \textbf{0.18--0.20} &
  95.9 & 95.9 & 0.18--0.20 &
  90.2 & 90.1 & 0.19--0.21 &
  24.3 & 19.9 & 0.13--0.15 \\
 ImageNet &
  \textbf{99.5} & \textbf{99.5} & \textbf{0.18--0.20} &
  95.2 & 95.1 & 0.18--0.20 &
  38.0 & 35.2 & \textbf{0.13--0.14} &
  14.7 & 12.7 & 0.15--0.17 \\
 tanh-Standard &
  99.2 & 99.2 & 0.19--0.22 &
  95.1 & 95.0 & 0.20--0.22 &
  92.4 & 92.3 & 0.20--0.22 &
  72.6 & 71.5 & 0.16--0.18 \\
 tanh-ImageNet &
  99.0 & 99.0 & 0.19--0.21 &
  \textbf{97.8} & \textbf{97.8} & 0.19--0.20 &
  \textbf{96.7} & \textbf{96.7} & 0.19--0.21 &
  90.9 & 90.8 & 0.18--0.20 \\
\midrule
\multicolumn{13}{c}{ResNet50} \\
\midrule
 Standard &
  \textbf{99.5} & \textbf{99.5} & 0.19--0.21 &
  16.0 & 14.5 & \textbf{0.14--0.15} &
  12.7 & 11.6 & 0.14--0.16 &
  10.9 & 9.5 & 0.15--0.17 \\
 ImageNet &
  98.5 & 98.5 & 0.19--0.21 &
  11.2 & 10.1 & 0.14--0.16 &
  10.0 & 8.9 & 0.15--0.17 &
  9.5 & 8.0 & 0.14--0.16 \\
 tanh-Standard &
  99.3 & 99.3 & 0.19--0.21 &
  93.3 & 93.3 & 0.18--0.20 &
  85.0 & 84.7 & 0.20--0.22 &
  27.6 & 25.4 & \textbf{0.13--0.14} \\
 tanh-ImageNet &
  99.0 & 99.0 & 0.19--0.21 &
  97.7 & 97.7 & 0.18--0.20 &
  96.6 & 96.6 & 0.18--0.20 &
  \textbf{93.3} & \textbf{93.2} & 0.18--0.20 \\
\bottomrule
\end{tabular}}
\end{table*}

\section{\uppercase{Results}}\label{sec:results}
In this section, we present the outcomes of our experiments on the
COVID-19 and MNIST tasks.

\subsection{Results on the COVID-19 Database}\label{sec:eval-covid}
\cref{tab:eval-covid} summarizes our results in performance (accuracy, F1-score) and practical privacy (empirical privacy leakage as $\Advinc$) over all privacy levels.

The non-private baseline at $\eps=\infty$ is dominated by the ImageNet variants, which reaches 95.9\% F1 for ResNet18.
The non-private tanh models generally show to perform less than their ReLU counterparts, making tanh less suited for non-private training.
Pneumonia pre-training does present an upgrade but smaller than ImageNet.
For the weakest privacy setting of $\eps=10$, we already find a steep utility-privacy trade-off of 16.2\%, when comparing the best F1-scores.
The ResNet18 ImageNet keeps the best performance.
However, the other tanh models now surpass their ReLU siblings.
The ResNet50 ReLU models struggle so much, that they do not provide a working model (0.0\% F1).
For both ResNet18 and ResNet50, the tanh-ImageNet now falls far behind.

Our results at $\eps=1$ present a substantial paradigm shift.
The ResNet18 ImageNet, our former winner, stops working at this level, with the only working ReLU model left being the ResNet18 Penumonia
Thus, at this stronger privacy setting, the tanh activation is essential to achieve feasible models for all variants.
We now see the tanh-Pneumonia delivering the best performing models with 72.5\% F1 on ResNet50 and 75.2\% accuracy on ResNet18.
From $\eps=10$ to $\eps=1$, we lose 6.9\% F1.
At $\eps=0.1$, the ResNet50 tanh-Pneumonia performs best with 69.4\% F1 and 73.0\% accuracy.
The ResNet18 sibling is still close in accuracy but falls further behind in F1.
On the ResNet50, the resulting trade-off of 3.1\% F1 compared to $\eps=1$ is less significant than before. 

Regarding privacy leakage given as $\Advinc$, we only see minor improvements when comparing non-private and private settings.
The initial observation is underlined by mean advantages across privacy levels: 0.23--0.25 ($\eps=\infty$), 0.22--0.23 ($\eps=10$), 0.20--0.22 ($\eps=1$), and 0.20--0.22 ($\eps=0.1$).
The differences are especially small in the better models. The lower performance models, however, offer less confidence in their predictions, making it harder for an attacker to correctly classify membership~\citep{rahman2018membership}.

\subsection{Results on the MNIST Database}\label{sec:eval-mnist}
We now refer to \cref{tab:eval-mnist} and omit the Pneumonia variants because there is no relation between the tasks.

The non-private performance trends show that MNIST is an easier task than COVID-19 detection because only one model falls short of reaching 99\% F1.
Furthermore, the tanh models are on par with the ReLu models even in the non-private setting.
Familiarly to the COVID-19 task, the ResNet18 ImageNet performs best at 99.5\% F1, but shares the crown with both Standard variants.
In private training, the ResNet50 ReLU models again start to fail at $\eps=10$.
They still achieve some F1 but scores of 14.5\% and 10.1\% can also be seen as unusable.

The other models, however, only see a smaller reduction in their performance than on COVID-19.
The tanh-ImageNet shines on MNIST, with the ResNet18 model achieving the best results at 97.8\% F1 and accuracy, which is just a 1.7\% reduction from $\eps=\infty$.
The ResNet50 counterpart closely follows at 97.7\%.
The tanh-ImageNet superiority carries over to $\eps=1$, where the ResNet18 tanh-ImageNet reigns at 96.7\% F1 and accuracy.
This again results in a marginal trade-off of 1.1\% from $\eps=10$ to $\eps=1$.
However, other than at $\eps=10$, the ImageNet on ReLU now also fails to deliver good F1 at 35.2\% on ResNet18.
The marginal trade-offs cannot be kept for $\eps=0.1$, where the ResNet50 tanh-ImageNet beats the ResNet18 sibling but now loses 3.4\% F1 and lands at 93.2\%.
The total F1 trade-off from $\eps=\infty$ to $\eps=0.1$ is still low at 6.3\% and considerably better than on the COVID-19 task.
F1 for ResNet18 Standard and ImageNet on ReLU decreases steeply at $\eps=0.1$, delivering rather unusable models.
Even the ResNet50 tanh-Standard now falls to a low of 25.4\% F1.

Shifting the view to our privacy analysis, non-private models on MNIST show a slightly lower proneness to MIAs than on COVID-19.
For models who only achieve low F1-scores, we again see lower $\Advinc$ values stemming from their corresponding low memorization~\citep{rahman2018membership}.
The better performing variants, on the other hand, again show just minor changes in leakage measurements.
When we exclude the misleading low-performing results with under 50\% F1, the following mean advantages confirm our point: 0.19--0.21 ($\eps=\infty$), 0.18--0.20 ($\eps=10$), 0.19--0.21 ($\eps=1$), and 0.17--0.19 ($\eps=0.1$).
We thus only find a small leakage difference between COVID-19 and MNIST, with COVID-19 models being slightly more prone in general.

\section{\uppercase{Discussion}}\label{sec:discussion}
\begin{figure*}[t]
  \centering
  \begin{minipage}{0.75\linewidth}
    \subcaptionbox{COVID-19\label{subfig:covid}}{\includegraphics[width=0.49\linewidth]{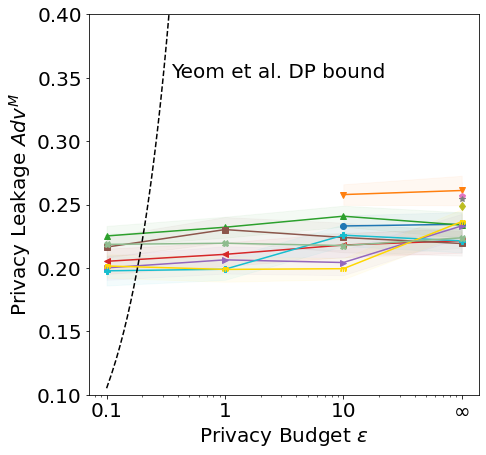}}
    \subcaptionbox{MNIST\label{subfig:mnist}}{\includegraphics[width=0.49\linewidth]{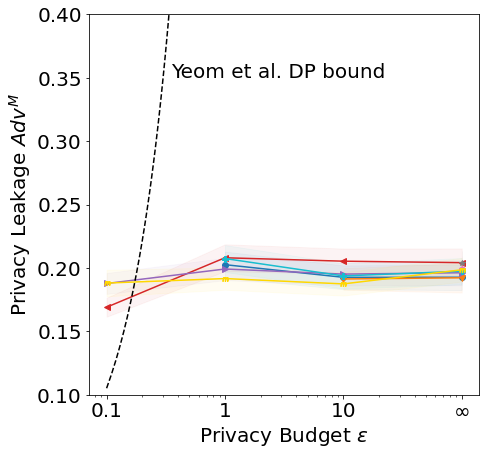}}
  \end{minipage}
  \begin{minipage}{0.24\linewidth}
      \raisebox{0.5in}{\includegraphics[width=1.0\linewidth]{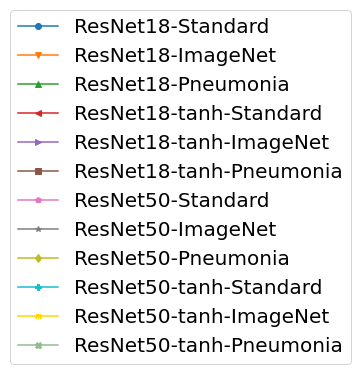}}
  \end{minipage}
  \vspace{0.1in}
  \caption{
    Empirical privacy leakage results from MIAs are given as our 95\% CI membership advantage ($\Advinc$) and plotted across the different privacy budgets.
    Model variants can be distinguished with the legend.
    We exclude data points with <50\% F1 because low performance disproportionately reduces leakage.
    A dotted line shows the DP bound from \citet{yeom2018privacy}.
  }\label{fig:graphs}
\end{figure*}

We now revisit the open gaps from the related work discussed in \cref{sec:related} and review the outcomes of our proposed solutions from \cref{sec:methods}.
We again refer to \cref{tab:eval-covid,tab:eval-mnist} for our evaluation results and to \cref{tab:eval-related} for a short and organized showcase of the compared methods from related work~\citep{muftuouglu2020differential,zhang2021feddpgan,ho2022dpcovid}.
In the following we evaluate our proposed improvements:

\textit{Datasets.}
    We achieve a more balanced data basis than before by utilizing undersampling and class weights.
    To better evaluate on the still skewed data, we add the F1-score metric.
    The advantage regarding accuracy is visible in the COVID-19 results, where both metrics differ regularly and F1 thus reveals models that perform better on the minority class COVID-19.
    That F1 better represents the underlying class distribution is further demonstrated by comparison to the more balanced MNIST dataset, where accuracy and F1-score are almost identical.
    
\textit{Privacy budgets.}
    In contrast to related work, we are able to achieve working COVID-19 detection models, while adhering to strong privacy budgets of $\eps \leq 1$.
    By additionally evaluating different architectures over multiple privacy levels, we deduce favorable architectural decisions for keeping good utility-privacy trade-offs in DP-SGD.
    Our findings for training private models are summarized in \cref{dis:models}.
    
\textit{Practical privacy.}
    By including an empirical study on practical privacy though MIAs, we gain insights into the relationship between DP and privacy leakage.
    In \cref{dis:practical} we derive the implications stemming from our empirical analysis.
    The results allow us to improve the utility-privacy trade-off while keeping the same practical privacy.

\subsection{Building Better DP-SGD Models}\label{dis:models}
With our experiments, we transfer the results from \citet{papernot2021tempered} to deeper and pre-trained networks and are able to confirm the tanh advantage over ReLU in low \eps DP-SGD training.
Our private models with strong \eps-guarantees of $\eps=1$ and $\eps=0.1$ rely on this change, while the non-private and less private models still prefer the ReLU activation function.

A commonality between our best performers is that they were subjected to pre-training.
While all best non-private models are pre-trained on ImageNet, this trend only continues in all private models on the MNIST database.
The same ImageNet-based models underperform on the COVID-19 task, when looking at settings of $\eps=1$ and $\eps=0.1$, which might be related to the different contents in both tasks.
On the COVID-19 task, we introduce task-specific pre-training on pneumonia images, that leads to superior performance in our most private settings.

We could not fully confirm that larger models perform worse in DP-SGD~\citep{papernot2021tempered}.
The ResNet50 especially wins at the most private setting of $\eps=0.1$. 
Model size, however, seems to play a role, when examining the earlier failure of the private ReLU models in the bigger ResNet50.

In summary, our results support the existing belief that model architectures should be specifically adjusted for private DP-SGD training, where established standards from non-private training do not necessarily provide the same advantages~\citep{papernot2021tempered,abadi2016deep}.
Examples are the switch from ReLU to tanh activation and the superiority of Pneumonia pre-training to ImageNet pre-training in the private COVID-19 models.

\subsection{Insights Regarding Practical DP}\label{dis:practical}
For this section, we refer to \cref{fig:graphs} that visualizes the results for our estimated privacy leakage from MIAs at the different $\eps$-budgets.

In both \cref{subfig:covid,subfig:mnist}, we include the ($\eps$,$\delta$)-DP bound on $\Advinc$ from \cref{eq:dpbound}, which is based on \citet{yeom2018privacy}.
The bound already surpasses our plotted maximum of 0.5 $\Advinc$ long before $\eps=1$, which shows the large discrepancy between the theoretically assumed worst-case and practice.
Simultaneously, no model actually trained for $\eps=0.1$ is able to conform to the calculated bound.
Such inconsistencies can also be found in related work~\citep{yeom2018privacy,jayaraman2019evaluating}.
As an explanation, \citet{yeom2018privacy} unveil that, in practice, the training set error distributions are not exactly Gaussian, sometimes leading to better attack performance than predicted by theory.
Even though COVID-19 and MNIST have rather opposing priors, where the former's classes are skewed and the latter's roughly balanced, we see the same inconsistencies in both evaluations.
Thus, the given theoretical bound does not seem reliable for deriving a limit on the real world threat in our case.

For both COVID-19 and MNIST, the leakage almost describes a flat line with just negligible changes over all privacy settings.
We spot a few outliers\footnote{On COVID-19 the outliers are both tanh-ImageNet models, which reduce their leakage from non-private to $\eps=10$, and the ResNet50 tanh-Standard doing the same from $\eps=10$ to $\eps=1$. There is also one outlier on MNIST, where the Resnet18-tanh-Standard improves privacy at $\eps=0.1$.}, that see a bigger drop in leakage risk, which however, is mainly attributed to their gravely lowered performance (>20\% F1 loss) and accordingly reduced memorization~\citep{rahman2018membership}.
Even the non-private models exhibit almost the same leakage as the private models and thus, including DP-guarantees does not imply the expected improvement to practical MIA proneness.
The plateau in privacy leakage can enable the use of lesser DP-guarantees, while still providing the same practical privacy
The MNIST models show to generally leak slightly less than on COVID-19, leading to stronger privacy needs for COVID-19.
The existing difference in MIA risk between COVID-19 and MNIST suggests, that privacy estimation can be an important tool for assessing task- and data-dependent threats from attacks.
Thus, such estimates can in turn support tuning trade-offs according to task-specific privacy needs.

The findings suggest room for utilizing weaker DP-guarantees on both tasks when defending against our MIA-specific setting.
Practical privacy is already strong in our less private and even non-private models.
We are thus able to improve the utility-privacy trade-off on both datasets at no practical privacy cost.
We could introduce even stronger guarantees to possibly further improve MIA defense, however, this would lead to an even bigger utility loss and in turn result in impractical performance.

We want to emphasize that there is still a need for strong theoretical privacy guarantees~\citep{nasr2021adversary}.
As stated in \cref{sec:back-repelling}, $\eps$-guarantees from DP limit the maximum amount of possible information leakage.
In actual attacks, however, the theoretical ceiling might differ notably from the practical threat as shown in this and other works presented in \cref{sec:related-practical}.
From this, we should not conclude that DP is unnecessary, since future adversaries could find better attacks that make an earlier empirical evaluation invalid.
To cover for such cases it is therefore advised to keep a reasonably strong privacy guarantee even when tuning for better trade-offs.
Thus, we would rather choose a COVID-19 model at $\eps=10$ than at $\eps=\infty$, even though both exhibit almost the same practical privacy levels.
The model at $\eps=10$ performs better than the one at $\eps=1$ and, in contrast to $\eps=\infty$, still provides a provable DP guarantee to limit future adversaries.

\section{\uppercase{Conclusion}}\label{sec:conclusion}
Within this piece of work, we close several open gaps in the field of private COVID-19 detection from X-ray images.
In comparison to related work on the topic, we improve data handling regarding imbalances, deliver a more robust privacy evaluation, and are the first to investigate the implications concerning practical privacy~\citep{muftuouglu2020differential,zhang2021feddpgan,ho2022dpcovid}.

We introduce a selection of yet untested architectural ML model choices to the COVID-19 task.
Through our evaluation, we are able to compare the setups in a common environment.
Since well-known practices from non-private training are not always transferable to DP-SGD training, it is important to gather a wide range of results for finding the best models.
We are therefore making a noticeable contribution by exploring a range of different architectures on the COVID-19 and MNIST tasks.

Our practical privacy analysis reveals that assessing attack-specific threats from black-box MIAs in a practical scenario helps finding appropriate privacy attributes and can thus improve the utility-privacy trade-off at no practical cost.
On both the COVID-19 and MNIST datasets, we found just minor improvements from the provided theoretical DP-guarantee regarding practical defense against our MIAs.
Instead, our tested models almost showed the same strong repelling properties across all privacy levels---even for non-private models.
By confirming this plateau for both datasets, we are able to reduce the required DP guarantees for both tasks without sacrificing attack-specific practical privacy.
Our attacks are slightly more successful on the COVID-19 task, showing that it needs stricter privacy than MNIST and that practical privacy analysis is important for identifying the task-specific initial MIA threat.

We still advocate the use of DP and would not recommend to risk publishing non-private COVID-19 detection models.
Instead, if justified by a practical privacy analysis, the \eps-guarantee can be tuned to a more favorable utility-privacy trade-off that through the inclusion of a reasonable DP-guarantee still limits the worst-case privacy leakage from future attacks.

As a brief outlook into possible future work, it would be beneficial to extend our evaluation by applying practical privacy analysis to more datasets, especially with different underlying tasks.
Another venture could be to derive best practices and ultimately a taxonomy regarding advantageous architectural decisions when training DP-SGD models.


\section*{\uppercase{Acknowledgments}} 
We thank our colleagues for insights on earlier drafts.
The authors acknowledge the financial support by the Federal Ministry of Education and Research of Germany and by the S\"achsische Staatsministerium f\"ur Wissenschaft Kultur und Tourismus in the program Center of Excellence for AI-research "Center for Scalable Data Analytics and Artificial Intelligence Dresden/Leipzig", project identification: ScaDS.AI.
Computations were done (in part) using resources of the Leipzig University Computing Centre.

\section*{\uppercase{Ethical Principles}}
All patient data originated from public sources provided for research purposes and was solely used within the limited scope of this work.

\bibliographystyle{apalike}
{\small
\bibliography{paper}}

\section*{\uppercase{Appendix}} 
\appendix
\section{\uppercase{Proof of Equation 3}}\label{sec:dpbound}
In the following \cref{thm:dp-bound}, we formally prove \cref{eq:dpbound}.
With the prerequisites of Experiment 1 and Definition 4 from \citet{yeom2018privacy}, we adapt their Theorem 1 and corresponding proof from $\eps$-DP to ($\eps$,$\delta$)-DP.
\begin{theorem}
	\label{thm:dp-bound}
    Let 
        $A$ be an ($\eps$,$\delta$)-differentially private learning algorithm, 
	    \A be a membership adversary, 
	    $\Advinc$ the membership advantage of \A, 
	    n be a positive integer, and 
	    D be a distribution over data points (x, y).
	Then we have:
	\[
	\Advinc(\A, A, n, \D) \le e^\eps - 1 + \delta.
	\]
\end{theorem}
\begin{proof}
    According to Definition 4 in~\cite{yeom2018privacy}, $\Advinc(\A,A,n,\D)$ can be expressed as the difference between \A's true and false positive rates
    \begin{equation} 
    \label{eq:altadv}
    \Advinc = \Pr[\A = 0 \mid b = 0] - \Pr[\A = 0 \mid b = 1],
    \end{equation}
    where $\Advinc$ is a shortcut for $\Advinc(\A,A,n,\D)$.

	Given $S = (z_1, \ldots, z_n) \sim \D^n$ and an additional point $z' \sim \D$, define $S^{(i)} = (z_1, \ldots, z_{i-1}, z', z_{i+1}, \ldots, z_n)$. Then, $\A(z',\model,n,\D)$ and $\A(z_i,\modeli,n,\D)$ have identical distributions for all $i \in [n]$, so we can write:
	\begin{align*}
	 	\Pr[\A = 0 \mid b = 0] &= 1 - \E_{S\sim\D^n}\left[\frac{1}{n}\sum_{i=1}^n \A(z_i, \model, n, \D)\right] \\
	 	\Pr[\A = 0 \mid b = 1] &= 1 - \E_{S\sim\D^n}\left[\frac{1}{n}\sum_{i=1}^n \A(z_i, \modeli, n, \D)\right]
	\end{align*}
	The above two equalities, combined with Equation~\ref{eq:altadv}, gives:
	\begin{equation}
	 	\label{eq:dpadv}
	 	\Advinc = \E_{S\sim\D^n}\left[\frac{1}{n}\sum_{i=1}^n \A(z_i, \modeli, n, \D) - \A(z_i, \model, n, \D)\right]
	\end{equation}
	
	Without loss of generality for the case where models reside in an
	infinite domain, assume that the models produced by $A$ come from the
	set $\{A^1, \ldots, A^k\}$. ($\eps$,$\delta$)-DP guarantees that for
	all $j \in [k]$,
	\[
	 	\Pr[\modeli = A^j] \le e^\eps\Pr[\model = A^j] + \delta .
	\]
	Using this inequality, we can rewrite and bound the right-hand side of Equation~\ref{eq:dpadv} as
	\begin{align*}
	 	&\sum_{j=1}^k \E_{S\sim\D^n}\Bigg[\frac{1}{n}\sum_{i=1}^n \Pr[\modeli = A^j] - \Pr[\model = A^j] \cdot \A(z_i, A^j, n, \D)\Bigg] \\
	 	& \le \sum_{j=1}^k \E_{S\sim\D^n}\left[(\delta + (e^\eps - 1) \Pr[\model = A^j]) \cdot \frac{1}{n}\sum_{i=1}^n \A(z_i, A^j, n, \D)\right],
	\end{align*}
	which is at most $e^\eps - 1 + \delta$ since $\A(z, A^j, n, \D) \le 1$ for any $z$, $A^j$, $n$, and $\D$.
\end{proof}


\end{document}